\documentclass[journal,onecolumn,draftclsnofoot]{IEEEtran}
\usepackage{graphicx}
\usepackage{amsfonts}
\usepackage{amssymb}
\usepackage{amsmath}
\usepackage{mathdots}
\usepackage{mathtools}
\usepackage{amsthm}
\usepackage{subfigure}

\usepackage{enumitem}
\usepackage{url}
\usepackage{dsfont}

\usepackage{multirow,tabularx}
\usepackage{flushend}
\usepackage[numbers,sort&compress]{natbib}

\usepackage{subfig}
\usepackage{epstopdf}

\usepackage{algorithm}
\usepackage{algorithmic}

\usepackage{wrapfig}

\usepackage{color}
\newcommand{\rstyle}[1]{{\textcolor{black}{#1}}}

\usepackage{tablefootnote}

\newtheorem{theorem}{Theorem}
\newtheorem{lemma}{Lemma}

\newtheorem{open problem}{Open Problem}

\begin{document}
\bstctlcite{IEEEexample:BSTcontrol}
\title{Online Metric Learning for Multi-Label Classification* \footnote{*The original version of this paper is published in AAAI 2020.}}

\author{Xiuwen Gong,
		Jiahui Yang,
        Dong Yuan,
        Wei Bao
\thanks{X. Gong, D. Yuan, W. Bao are with the Faculty of Engineering,
The University of Sydney. J. Yang is with the School of Computer Science and Engineering, The University of New South Wales, Australia (E-mails: xiuwen.gong@sydney.edu.au; jiahuiyang0@gmail.com; dong.yuan@sydney.edu.au; wei.bao@sydney.edu.au)}
}


\maketitle

\begin{abstract}
Existing research into online multi-label classification, such as online sequential multi-label extreme learning machine (OSML-ELM) and stochastic gradient descent (SGD), has achieved promising performance. However, these works do not take label dependencies into consideration and lack a theoretical analysis of loss functions. Accordingly, we propose a novel online metric learning paradigm for multi-label classification to fill the current research gap. Generally, we first propose a new metric for multi-label classification which is based on $k$-Nearest Neighbour ($k$NN) and combined with large margin principle. Then, we adapt it to the online settting to derive our model which deals with massive volume ofstreaming data at a higher speed online. Specifically, in order to learn the new $k$NN-based metric, we first project instances in the training dataset into the label space, which make it possible for the comparisons of instances and labels in the same dimension. After that, we project both of them into a new lower dimension space simultaneously, which enables us to extract the structure of dependencies between instances and labels. Finally, we leverage the large margin and $k$NN principle to learn the metric with an efficient optimization algorithm. Moreover, we provide theoretical analysis on the upper bound of the cumulative loss for our method. Comprehensive experiments on a number of benchmark multi-label datasets validate our theoretical approach and illustrate that our proposed online metric learning (OML) algorithm outperforms state-of-the-art methods.
\end{abstract}


\IEEEpeerreviewmaketitle

\section{Introduction}

Real-world applications often involve in generating massive volume of streaming data at an unprecedented high speed. Many researchers have focused on data classification to help customers or users get better searching results, among which `online multi-label classification' which means each instance can be assigned multiple labels is very useful in some applications. For example, in the web-related applications, Twitter, Facebook and Instagram posts and RSS feeds are attached with multiple essential forms of categorization tags \cite{DBLP:journals/jmlr/ZhangGH10} . In the search industry, revenue comes from clicks on ads embedded in the result pages. Ad selection and placement can be significantly improved if ads are tagged correctly. There are many other applications, such as object detection in video surveillance \cite{DBLP:conf/www/PopoviciWG14} and image retrieval in dynamic databases \cite{DBLP:conf/icmcs/DongB03}.

In the development of multi-label classification \cite{DBLP:journals/ml/TsoumakasZZ12,DBLP:journals/csur/GibajaV15}, one challenge that remains unsolved is that most multi-label classification algorithms are developed in an off-line mode \cite{DBLP:journals/ml/ChengH09,DBLP:conf/nips/ChenL12,DBLP:conf/wsdm/BabbarS17,DBLP:journals/jmlr/LiuT17,DBLP:journals/jmlr/ZhouTPT19,DBLP:journals/pami/LiuXTZ19}. These methods assume that all data are available in advance for learning. However,there are two major limitations of developing multi-label methods under such an assumption: firstly, these methods are impractical for large-scale datasets, since they require all datasets to be stored in memory; secondly, it is non-trivial to adapt off-line multi-label methods to the sequential data. In practice, data is collected sequentially, and data that is collected earlier in this process may expire as time passes.
Therefore, it is important to develop new multi-label classification methods to deal with streaming data.

Several online multi-label classification studies have recently been developed to overcome the above-mentioned limitations. For example, online learning with accelerated nonsmooth stochastic gradient (OLANSGD) \cite{DBLP:conf/icassp/ParkC13} was proposed to solve the online multi-label classification problem. Moreover, the online sequential multi-label extreme learning machine (OSML-ELM) \cite{DBLP:journals/evs/VenkatesanEDPW17} is a single-hidden layer feed-forward neural network-based learning technique. OSML-ELM classifies the examples by their output weight and activation function. Unfortunately, all of these online multi-label classification methods lack an analysis of loss function and disregard label dependencies. Many studies \cite{DBLP:conf/icml/DembczynskiCH10,DBLP:journals/ml/ReadPHF11,DBLP:conf/nips/BhatiaJKVJ15,DBLP:conf/icml/YenHRZD16,DBLP:journals/jmlr/LiuTM17} have shown that multi-label learning methods that do not capture label dependency usually achieve degraded prediction performance. This paper aims to fill these gaps.

$k$-Nearest Neighbour ($k$NN) algorithms have achieved superior performance in various applications \cite{DBLP:conf/eccv/DengBLF10}. Moreover, experiments show that distance metric learning on single-label prediction can improve the prediction performance of $k$NN. Nevertheless, there are two problems associated with applying a $k$NN algorithm to an online multi-label setting. Firstly, naive $k$NN algorithms do not consider label dependencies. Secondly, it is non-trivial to learn an appropriate metric for online multi-label classification.

To break the bottleneck of $k$NN, we here propose a novel multi-label learning paradigm for multi-label classification. More specifically, we project instances and labels into the same embedding space for comparison, after which we learn the distance metric by enforcing the constraint that the distance between embedded instance and its correct label must be smaller than the distance between the embedded instance and other labels. Thus, two nearby instances from different labels will be pushed further. Moreover, an efficient optimization algorithm is proposed for the online multi-label scenario. In theoretical terms, we analyze the upper bound of cumulative loss for our proposed model. A wide range of experiments on benchmark datasets corroborate our theoretical results and verify the improved accuracy of our method relative to state-of-the-art approaches.

The remainder of this paper is organized as follows. We first describe the related work, the online metric learning for multi-label classification and the optimization algorithm. Next, we introduce the upper bound of the loss function. Finally, we present the experimental results and conclude this paper.

\section{Related Work}
\label{Related_Work}
Existing multi-label classification methods can be grouped into two major categories: namely, \textit{algorithm adaptation} (AA) and \textit{problem transformation} (PT). AA extends specific learning algorithms to deal with multi-label classification problems. Typical AA methods include \cite{Zhang:2006:MNN:1159162.1159294,conf/ijcai/BrinkerH07,DBLP:journals/ml/ZhouTHM19}. Moreover, PT methods such as that developed by \cite{conf/nips/HsuKLZ09}, transform the learning task into one or more single-label classification problems. However, all of these methods assume that all data are available for learning in advance. These methods thus incur prohibitive computational costs on large-scale datasets, and it is also non-trivial to apply them to sequential data.

The state-of-the-art approaches to online multi-label classification have been developed to handle sequential data. These approaches can be divided into two key categories: \textsl{Neural Network} and \textsl{Label Ranking}. Neural Network approaches are based on a collection of connected units or nodes, referred to as artificial neurons. Each connection between artificial neurons can transmit the signal from one neuron to another. The artificial neuron that receives the signal can process it and then transmit signal to other artificial neurons. Moreover, label ranking, another popular approach to multi-label learning, involves a set of ranking functions being learned to order all the labels such that relevant labels are ranked higher than irrelevant ones.

From the neural network perspective, \citeauthor{DBLP:journals/air/DingZZXN15} \cite{DBLP:journals/air/DingZZXN15} developed a single-hidden layer feedforward neural network-based learning technique named ELM. In this method, the initial weights and the hidden layer bias are selected at random, and the network is trained for the output weights to perform the classification. Moreover, \citeauthor{DBLP:journals/evs/VenkatesanEDPW17} \cite{DBLP:journals/evs/VenkatesanEDPW17} developed the OSML-ELM approach, which uses ELM to handle streaming data. OSML-ELM uses a sigmoid activation function and outputs weights to predict the labels. In each step, the output weight is learned from the specific equation. OSML-ELM converts the label set from bipolar to unipolar representation in order to solve multi-label classification problems.

Some other existing approaches are based on label ranking, such as OLANSGD \cite{DBLP:conf/icassp/ParkC13}. In the majority of cases, ranking functions are learned by minimizing the ranking loss in the max margin framework. However, the memory and computational costs of this process are expensive on large-scale datasets. Stochastic gradient decent (SGD) approaches update the model parameters using only the gradient information calculated from a single label at each iteration. OLANSGD minimizes the primal form using Nesterov's smoothing, which has recently been extended to the stochastic setting.

However, none of these methods analyze the loss function, and all of them fail to capture the interdependencies among labels; these issues have been proved to result in degraded prediction performance. Accordingly, this paper aims to address these issues.

\begin{table}
	\small
	\centering
	\begin{tabular}{|p{0.22\columnwidth}|p{0.68\columnwidth}|}
		\hline
		\textbf{Notation}   & \textbf{Definition}                   \\ \hline
		$t$    &  the round of algorithm  \\ \hline
		
		$x_t$  &  an instance presented on round t \\ \hline
		
		$y_t$  &  corresponding label vector to $x_t$      \\ \hline
		
		$x$ &  nearest neighbour instance to $x_t$  \\ \hline
		
		$y$ & corresponding output of $x$  \\ \hline
		
		$X$ & initialized input matrix  \\ \hline
		
		$Y$ & corresponding output matrix  \\ \hline
		
		$n$ & the number of instances   \\ \hline
		
		$p$ & the number of features   \\ \hline
		
		$q$ & the number of labels   \\ \hline
		
		$d$ & the dimension of the new projection space   \\ \hline
		
		$V_t , P_t$ & projection matrix on round $t$ \\ \hline
		
		$m , M$ & lower bound and upper bound of $\lambda_t$ \\ \hline
		
		$\langle A, B \rangle_F$ & Frobenius inner product of $A$ and $B$ \\ \hline
		
		$||\cdot||_1$ & $l_1$ norm \\ \hline
		
		$||\cdot||_2$ & $l_2$ norm \\ \hline
		
		$||\cdot||_F$ & Frobenius norm \\ \hline

	\end{tabular}
	\vspace{-2mm}
	\caption{Summary of Notations}
	\label{tb:notations1}
\end{table}

\section{Our Proposed Method}
\label{OML}
\subsection{Notations}
We denote the instance presented to the algorithm on round $t$ by $x_t \in \mathbb{R}^{p \times 1}$, and the label by $y_t \in \{0, 1\}^{q \times 1}$, and refer each instance-label pair as an example. Suppose that we initially have $n$ examples in memory, denoted by $D = \{(x_i, y_i)\}^n_{i=1}$.  $(x, y) \in D$ is a nearest neighbour to $x_t$. The initialized instance matrix is denoted as $X \in \mathbb{R}^{n\times p} $ and the correspond output matrix is denoted as $Y \in \{0,1\}^{n\times q} $. $t$ is a positive integer. $||\cdot||_F$ is Frobenius norm. $V_t = (v_1,v_2,...,v_d) \in \mathbb{R}^{q \times d} (d < q)$ is projection matrix which maps each output vector $y_t$ ($q$ dimension) to $V^T y_t$ ($d$ dimension). Let $P \in \mathbb{R}^{p \times q}$ also be the projection matrix. Each input vector $x_t$ ($p$ dimension) is projected to $V^TP^Tx_t$ ($d$ dimension). Then $x_t$ and $y_t$ can be compared in the projection space($d$ dimension). Notations are summarized in Table \ref{tb:notations1}.


\subsection{Online Metric Learning}
Inspired by \citeauthor{DBLP:conf/nips/HsuKLZ09} \cite{DBLP:conf/nips/HsuKLZ09},who stated that each label vector can be projected into a lower dimensional label space, which is deemed as encoding, we propose the following large-margin metric learning approach with nearest neighbor constraints to learn projection. If the encoding scheme works well, the distance between the codeword of $x_t$, $(V^TP^Tx_t)$, and $y_t$, $(V^Ty_t)$, should tend to be 0 and less than the distance between codeword $x_t$ and any other output $V^Ty$. The following large margin formulation is then presented to learn the projection matrix $V$:
\begin{equation}\label{2}
\begin{split}
&argmin_{V\in \mathbb{R}^{q\times d}}  \frac{1}{2}||V||^2_F + \xi_t \\
&s.t.\quad ||V^{T}P^{T}x_t-V^Ty_t||^2_2 + \Delta(y_t,y) - \xi_t\\
&\quad \leq ||V^{T}P^{T}x_t-V^Ty||^2_2, \forall t \in \{1, 2, \cdots\}
\end{split}
\end{equation}
The constraints in Eq.(\ref{2}) guarantee that the distance between the codeword of $x_t$ and the codeword of $y_t$ is less than the distance between the codeword of $x_t$ and codeword of any other output. To give Eq.(\ref{2}) more robustness, we add loss function $\Delta(y_t, y)$ as the margin. The loss function is defined as $\Delta(y_t, y) = ||y_t - y||_1$, where $||\cdot||_1$ is the $l_1$ norm. After that, we use Euclidean metric to measure the distances between instances $x_t$ and $x$ and then learn a new distance metric, which improves the performance of $k$NN and also captures label dependency.

To retain the information learned on the round $t$, we apply above large margin formulation into online setting. Thus, we have to define the initialization of the projection matrix and the updating rule. We initialize the projection matrix $V_1$ to a non-zero matrix and set the new projection matrix $V_{t+1}$ to be the solution of the following constrained optimization problem on round $t$.
\begin{equation}\label{3}
\begin{split}
&V^T_{t+1} = argmin_{V\in \mathbb{R}^{q\times d}} \frac{1}{2} ||V^T - V^T_t||^2_F\\
&\quad s.t. \quad l(V;(x_t,y_t)) = 0
\end{split}
\end{equation}
The loss function is defined as following:
\begin{equation}\label{4}
\begin{split}
l(V;(x_t, y_t)) = &\max\{0, \Delta(y_t, y)\!\!-\!\! (||V^TP^Tx_t \!- \!\!\! V^Ty||^2_2 \\
&- ||V^TP^Tx_t - V^Ty_t||^2_2)\}
\end{split}
\end{equation}
where the matrix $P$ is learned through the following formulation:
\begin{equation*}
argmin_{P\in \mathbb{R}^{p\times q}}  \frac{1}{2}||P^{T}X^{T} - Y^{T}||^2_F
\end{equation*}
Define the loss function on round $t$ as
\begin{equation}\label{8}
\begin{split}
l_t(V_t;(x_t, y_t)) = &\max\{0, \Delta(y_t, y)\!\!-\!\! (||V^T_tP^Tx_t \!- \!\!\! V^T_ty||^2_2 \\
&- ||V^T_tP^Tx_t - V^T_ty_t||^2_2)\}
\end{split}
\end{equation}

When loss function is zero on round $t$, $V_{t+1} = V_t$. In contrast, on those rounds where the loss function is positive, the algorithm enforces $V_{t+1}$ to satisfy the constraint $l_{t+1}(V_{t+1};(x_{t+1},y_{t+1})) = 0$ regardless of the step-size required. This update rule requires $V_{t+1}$ to correctly classify the current example with a sufficient high margin and $V_{t+1}$ have to stay as closed as $V_t$ to retain the information learned on the previous round.

\subsection{Optimization}
The optimization of Eq.(\ref{3}) can be shown by using standard tools from convex optimization \cite{Boyd:2004:CO:993483}. If $l_t = 0$ then $V_t$ itself satisfies the constraint in Eq.(\ref{3}) and is clearly the optimal solution. Therefore, we concentrate on the case where $l_t > 0$. Firstly, we define the Lagrangian of the optimization problem in Eq.(\ref{3}) to be,
\begin{equation}\label{5}
\begin{split}
L = &\frac{1}{2}||V^T-V^T_t||^2_F + \lambda(\Delta(y_t, y) \\
&-\!(||V^TP^Tx_t - V^Ty||^2_2 - ||V^TP^Tx_t - V^Ty_t||^2_2))
\end{split}
\end{equation}
where the $\lambda$ is a Lagrange multiplier.

Setting the partial derivatives of $L$ with respect to the elements of $V^T$ to zero gives
\begin{equation*}\label{new1}
\begin{split}
0 = \frac{\partial L}{\partial V^T} = &V^T\!\!-\! V^T_t\!\!-\! 2V^T\lambda((P^Tx_t\!\!-\! y)(P^Tx_t\! -\! y)^T\! \\
&- (P^Tx_t\!\! -\!\! y_t)(P^Tx_t\!-\! y_t)(P^Tx_t\! -\!y_t)^T) \\
\end{split}
\end{equation*}
from this equation, we can get that
\begin{equation*}\label{new2}
\begin{split}
V^T = &V^T_t(I\!\! -\!\! 2 \lambda ((P^Tx_t - y)(P^Tx_t -y)^T\\
&- (P^Tx_t - y_t)(P^Tx_t - y_t)^T))^{-1}
\end{split}
\end{equation*}
in which $I$ stands for an identity matrix.

Inspired by \cite{IMM2012-03274}, we use an approximation form of $V^T$ to make it easier for following calculation.
\begin{equation}\label{6}
\begin{split}
\bar{V}^T = &V^T_t(I\!\! +\!\! 2 \lambda ((P^Tx_t - y)(P^Tx_t -y)^T\\
&- (P^Tx_t - y_t)(P^Tx_t - y_t)^T))
\end{split}
\end{equation}
Define $Q = V_tV^T_t$, $A = (P^Tx_t - y)(P^Tx_t - y)^T - (P^Tx_t - y_t)(P^Tx_t - y_t)^T$. Plugging the approximation formula Eq.(\ref{6}) back into Eq.(\ref{5}), we get a cubic function $f(\lambda) = a\lambda^3 + b\lambda^2 + c\lambda$, $\lambda \in \mathbb{R}$, where
\begin{equation*}\label{new3}
\begin{split}
a = &4(P^Tx_t - y_t)^TA^TQA(P^Tx_t-y_t) \\
&- (P^Tx_t-y)^TA^TQA(P^Tx_t-y)\\
b = &2(||V^T_tA||^2_F - (P^Tx_t - y_t)^TQA(P^Tx_t - y_t) \\
&- (P^Tx_t - y_t)^TA^TQ(P^Tx_t - y_t) \\
&+ (P^Tx_t - y)^TQA(P^Tx_t - y) \\
&+ (P^Tx_t - y)^TA^TQ(P^Tx_t -y))\\
c = &(P^Tx_t\!\! - y)^TQ(P^Tx_t \!\!- y)\!\! -\!\! (P^Tx_t \!\!- y)^TQ(P^Tx_t\!\! - y) \\
&+ \Delta(y_t, y)\\
\end{split}
\end{equation*}

If $f(\lambda)$ is non-monotonic function when $\lambda > 0$, let $\beta > 0$ to be the maximum point of $f(\lambda)$.
We obtain,
\begin{equation}\label{7}
\begin{split}
\lambda_t=
\begin{cases}
m & if \quad f'(\lambda) < 0 \quad and\quad \lambda > 0, \quad \beta < m  \\
\beta & if \quad m < \beta < M \\
M & if\quad f'(\lambda) > 0 \quad and \quad \lambda > 0, \quad \beta > M
\end{cases}
\end{split}
\end{equation}
where $m, M \in \mathbb{R}$, $0 < m < M$

\renewcommand{\algorithmicrequire}{\textbf{Input:}}
\renewcommand{\algorithmicensure}{\textbf{Output:}}
\begin{algorithm}[h!]
	\caption{\label{alg:1} Online Metric Learning for Multi-Label Classification}
	\begin{algorithmic}[1]
		\STATE Set $V_1 $ to a non-zero matrix
		\STATE Initialize $D = \{(x_i, y_i)\}_{i=1}^n$
		\FOR{$t= 1,2,\ldots,$ }
		\STATE Receive pairwise instances: $(x_t, y_t)$
		\STATE Find the Nearest Neighbour $(x, y) \in D$
		\STATE Compute loss $l_t$ by Eq.(\ref{8})
		\IF {$l_t > 0$}
		\STATE Set $\lambda_t$ as Eq.(\ref{7})
		\STATE Update $V^T = V^T_t(I - 2 \lambda_t A)^{-1}$
		\ELSE
		\STATE $V_{t+1} = V_t$
		\ENDIF
		\STATE Append current instances into $D$
		\ENDFOR
	\end{algorithmic}
\end{algorithm}

Algorithm \ref{alg:1} provides detail of optimization. We denote the loss suffered by our algorithm on round $t$ by $l_t$.

We focus on the situation when $l_t > 0$. The optimal solution comes from the one satisfying $\partial L / \partial V = 0$, $\partial L / \partial \lambda = 0$.  Based on the derivation, $V_{t+1}$ can be update by $V^T_{t+1} = V^T_t(I - 2 \lambda_t A)^{-1}$, where $A = (P^Tx_t - y)(P^Tx_t -y)^T - (P^Tx_t - y_t)(P^Tx_t - y_t)^T $.

Inspired by metric learning \cite{journals/ftml/Kulis13}, we use the learned metric to select $k$ nearest neighbours from $D$ for each testing instance, and conduct the predictions based on these $k$ nearest neighbours. The equation of the distance between codeword $x_j$ and $x_t$ in the embedding space can be computed as $(P^Tx_j-P^Tx_t)^TQ(P^Tx_j-P^Tx_t)$.

\begin{table}
	\caption{\rstyle{Training Time Complexities of Each Iteration and Testing Time Complexities of Each Testing Instance for all methods.}}
	\label{Tab:timeComplexity}
	\centering
	\small
	\begin{tabular}
		{|c|c|c|} \hline
		Method & Training Time & Testing Time \\ \hline
		OSML-ELM & $\mathcal{O}(np^2q)$& $\mathcal{O}(npq)$\\ \hline
		OLANSGD   &  $\mathcal{O}(n\log n)$ & $\mathcal{O}(pq)$ \\ \hline
		$k$NN  & -  & $\mathcal{O}(np)$ \\  \hline
		OML  &  $\mathcal{O}(np + dpq)$ &$\mathcal{O}(dpq+ nd)$  \\   \hline
	\end{tabular}
\end{table}

\subsection{Computational Complexity Analysis}

We compare the time complexities of our proposed method (i.e. OML) with three popular methods, which are OSML-ELM \citep{DBLP:journals/evs/VenkatesanEDPW17}, OLANSGD \citep{DBLP:conf/icassp/ParkC13} and $k$NN \citep{DBLP:conf/eccv/DengBLF10}.

The training time of OML is dominated by finding the nearest neighbour of each training instance and computing the loss $l_t$ in Eq.(\ref{8}). It takes $np$ time to search for the nearest neighbour from the training dataset while computing the loss with two projections embedded takes $dpq$ time. Thus, the time complexity is $\mathcal{O}(np + dpq)$. 

We analyze the testing time for each testing instance. The testing time of $k$NN involves the procedures of searching for the $k$ neareast neighbours of a testing instance from the training dataset which takes $np$ time. Our proposed PL-LMNN performs prediction in a similar way but differs in the additional procedure of projecting all instances into the embedding space of $d$ dimensions before searching for the neareast neighbours, therefore the testing time complexity of OML is $\mathcal{O}(dpq + nd)$. 

Moreover, training time complexities of each iteration and testing time complexities of each testing instance for other methods are listed in Table \ref{Tab:timeComplexity} for comparisions.

From Table~\ref{Tab:timeComplexity}, we can easily conclude that the training time complexity of OML in each iteration is lower than that of OSML-ELM and OLANSGD with respect to the number of training data $n$, which is usually much larger than the number of features $p$ and the number of labels $q$. Besides, the training time complexity of PL-$k$NN is denoted by $'-'$ as it has no training process. 

Moreover, for the testing time complexity, OML is lower than OSML-ELM and $k$NN with respect to the number of training data $n$. In addition, the reduced dimension $d$ of the new projected space is much smaller than the number of features as well as the number of labels. OLANSGD is the fastest in predicting among all methods, mainly because it performs prediction only by computing the label scores based on the learned model parameter. 

\section{Loss Bound}
\label{loss_bound}
Following the analysis in \cite{DBLP:journals/jmlr/CrammerDKSS06}, we state the upper bounds for our online metric learning algorithm.  Let $U = (u_1,u_2,...,u_d) \in \mathbb{R}^{q \times d} (d < q)$ be an arbitrary matrix. We use the approximate form given in Eq.(\ref{6}) to replace $V^T$.

\begin{lemma}
	\label{lemma_1}
	Let $\lambda_t$ as defined in Eq.(\ref{7}), $V = (v_1,v_2,...,v_d) \in \mathbb{R}^{q \times d} (d < q)$, $V_1$ is a non-zero matrix. The following bound holds for any $U \in \mathbb{R}^{q \times d} (d < q)$
	\begin{equation*}
	\begin{split}
	||V_1 - U||^2_F - ||V_{T+1} - U||^2_F \leq ||V_1 - U||^2_F
	\end{split}
	\end{equation*}
\end{lemma}

\begin{proof}
	Define $\Psi_t = ||V_t - U||^2_F - ||V_{t+1} - U||^2_F$, this lemma is proved by summing $\Psi_t$ over all $t$ in $1, \dots, T$ and the bounding of this sum is obviously as followed,
	\begin{equation*}
	\begin{split}
	\sum_{t=1}^T\Psi_t &= ||V_1 - U||^2_F\! - \!||V_{T+1} - U||^2_F
	\!\!\leq ||V_1 - U||^2_F
	\end{split}
	\end{equation*}
\end{proof}

\begin{lemma}
	\label{lemma_2}
	Assume there exists some $U$ such that $4\lambda_t \langle U, A^TV_t \rangle_F - 4 \lambda_t^2 ||A^TV_t|| \geq 5\lambda_t \langle V_t, A^TV_t \rangle_F + \frac{qc^2 \lambda_t}{(||P||^2_Fr + q)} $, $\forall t \in \{1,2,..., T\}$. Let $V = (v_1,v_2,...,v_d) \in \mathbb{R}^{q \times d} (d < q)$. $\lambda_t$ as defined as in Eq.(\ref{7}). $V_1$ is a non-zero matrix. $c$ is defined in the proof Eq.(\ref{8}). We bound cumulative $||V_t||^2_F$ as follows,
	\begin{equation*}
	\sum_{t=1}^T ||V_t||^2_F  \leq \frac{||V_1 - U||^2_F}{m \cdot c^2} - \frac{q \cdot T}{(||P||^2_Fr + q)}
	\end{equation*}
\end{lemma}
\begin{proof}
	By using the operation of Frobenius norm,
	\begin{equation*}
	||A+B||^2_F = ||A||^2_F + ||B||^2_F + 2\langle A, B\rangle_F
	\end{equation*}
	where $\langle \cdot \rangle_F$ is the Frobenius inner product, we can get
	\begin{equation*}
	\begin{split}
	\Psi_t &= ||V_t - U||^2_F - ||V_{t+1} - U||^2_F \\
	&= ||V_t - U||^2_F - ||(I +2\lambda_tA)^TV_t -U||^2_F \\
	&= ||V_t - U||^2_F - ||V_t - U||^2_F -4\lambda_t^2||A^TV_t||^2_F \\
	&\quad - 4\lambda_t\langle V_t - U, A^TV_t \rangle_F \\
	&= -4\lambda_t \langle V_t, A^TV_t \rangle_F + 4\lambda_t \langle U, A^TV_t \rangle_F \\
	&\quad - 4\lambda_t^2||A^TV_t||^2_F
	\end{split}
	\end{equation*}
	Using the assumption in Lemma \ref{lemma_2}, we can get that $\Psi_t \geq \lambda_t \langle V_t, A^TV_t \rangle_F + \frac{qc^2 \lambda_t}{(||P||^2_Fr + q)}$.
	where $A = (P^Tx_t - y)(P^Tx_t -y)^T - (P^Tx_t - y_t)(P^Tx_t - y_t)^T$. It is clearly that $A$ is a symmetric matrix. We take the SVD of $A$ as $A = \bar{U}\bar{A}\bar{U}^T$, then using the minimum non-negative singular value of $A$ to replace the non-positive element in matrix $\bar{A}$, and denote approximation form of matrix $A$ as $\hat{A}$. Apparently, $\hat{A}$ is a non-negative symmetric matrix. Furthermore, by using definition of Frobenius inner product $\langle A, B \rangle_F = Trace(A^TB)$, where $Trace(A) = \sum_{i=1}^{n}a_{ii}$, we can get that
	\begin{equation*}
	\begin{split}
	\langle V_t, \hat{A}^TV_t \rangle_F &= Trace(V^T_t\hat{A}^TV_t) \\
	&= Trace(V_t^T(\hat{A}^{\frac{1}{2}T})\hat{A}^{\frac{1}{2}}V_t) \\
	&= Trace((\hat{A}^{\frac{1}{2}}V_t)^T\hat{A}^{\frac{1}{2}}V_t) \\
	&= ||\hat{A}^{\frac{1}{2}}V_t||^2_F
	\end{split}
	\end{equation*}
	Taking the SVD of $\hat{A}^\frac{1}{2}$ as $\hat{A}^\frac{1}{2} = U^*A^*U^{*T}$. Since matrix $U^*$ is a unitary matrix, then $||U^*B||^2_F = ||B||^2_F$, $\forall B \in \mathbb{R}^{q \times q}$. Let $c$ be the minimum singular value of $A^*$, getting that
	\begin{equation}
	\begin{split}
	||\hat{A}^{\frac{1}{2}}V_t||^2_F &= ||U^*A^*U^{*T}V_t||^2_F \\
	&= ||A^*U^{*T}V_t||^2_F \\
	&\geq ||cIU^{*T}V_t||^2_F \\
	&\geq c^2||V_t||^2_F
	\end{split}
	\end{equation}
	where $I$ is an identity matrix. Now, we get that,
	\begin{equation*}
	\Psi_t \geq c^2\lambda_t||V_t||^2_F + \frac{qc^2 \lambda_t}{(||P||^2_Fr + q)}
	\end{equation*}
	By summing both side of inequality on $t$ over all $t$ in $1, \dots, T$, and using that $m \leq \lambda_t \leq M$, gives that
	\begin{equation*}
	\begin{split}
	\sum_{t=1}^T c^2 \cdot m ||V_t||^2_F + \frac{T \cdot qc^2m}{(||P||^2_Fr + q)} \leq ||V_1 - U||^2_F \\
	\end{split}
	\end{equation*}
	Then, we can get that
	\begin{equation*}
	\begin{split}
	\sum_{t=1}^T ||V_t||^2_F  \leq \frac{||V_1 - U||^2_F}{m \cdot c^2} - \frac{q \cdot T}{(||P||^2_Fr + q)} \\
	\end{split}
	\end{equation*}
	Lemma \ref{lemma_2} has been proved.
\end{proof}
Based on the Lemma \ref{lemma_2}, we provide following theorem.

\begin{theorem}
	Let $(x_1, y_1),\dots,(x_T,y_T)$ be a sequence of examples where $x_t \in \mathbb{R}^{p \times 1}$ and $y_t \in \{0, 1\}^{q \times 1}$. $V_t = (v_1,v_2,...,v_d) \in \mathbb{R}^{q \times d} (d < q)$ is projection matrix, $q$ is in $\mathbb{R}^n$. $V_1$ is a non-zero matrix .$U \in \mathbb{R}^{q \times d} (d < q)$. Let $r$ be the upper bound of $||x_t||^2_2$. Under the assumption of Lemma \ref{lemma_2}, the cumulative loss suffered on the sequence is bounded as follows,
	\begin{equation*}
	\begin{split}
	\sum_{t=1}^T l_t \leq \frac{||V_1 - U||^2_F(||P||^2_F \cdot r + q)}{m \cdot c^2}
	\end{split}
	\end{equation*}
\end{theorem}

\begin{proof}
	By using Eq.(\ref{8}), we get that
	\begin{equation*}
	\begin{split}
	l_t &\leq \Delta(y_t, y) + ||V^T_tP^Tx_t - V^T_ty_t||^2_2 \\
	\end{split}
	\end{equation*}
	and,
	\begin{equation*}
	\begin{split}
	||P^Tx_t - y_t||^2_2 &\leq ||P^Tx_t||^2_2 + ||y_t||^2_2 \leq ||P^T||^2_F\cdot r + q
	\end{split}
	\end{equation*}
	
	Since $\Delta(y_t, y)$ is defined as $l_1$ norm, therefore $\Delta(y_t, y)$ is bounded by $q$. we can get $y$ is bounded by $q$ as well. By using Lemma \ref{lemma_2}, we can get,
	\begin{equation*}
	\begin{split}
	\sum_{t=1}^Tl_t &\leq T\cdot q + \sum_{t=1}^T||V_t||^2_F \cdot ||P^Tx_t - y_t||^2_2\\
	&\leq T\cdot q + \sum_{t=1}^T ||V_t||^2_F \cdot (||P||^2_F\cdot r + q) \\
	&\leq T\cdot q + (\frac{||V_1 - U||^2_F}{m \cdot c^2} - \frac{q \cdot T}{(||P||^2_Fr + q)} )(||P||^2_F\cdot r + q) \\
	&\leq \frac{||V_1 - U||^2_F(||P||^2_F \cdot r + q)}{m \cdot c^2}
	\end{split}
	\end{equation*}
\end{proof}

Therefore, the cumulative loss is bounded by $\frac{||V_1 - U||^2_F(||P||^2_F \cdot r + q)}{m \cdot c^2}$. As $l_t$ is bounded, it guarantees the performance of our proposed model for unseen data.
\section{Experiments}
\label{Experiments}
In this section, we conduct experiments to evaluate the prediction performance of the proposed OML for online multi-label classification, and compare it with several state-of-the-art methods. All experiments are conducted on a workstation with 3.20GHz Intel CPU and 16GB main memory, running the Windows 10 platform.
\subsection{Datasets}
We conduct experiments on eight benchmark datasets: Corel5k \citep{DBLP:conf/eccv/DuyguluBFF02}, Enron \footnote{http://bailando.sims.berkeley.edu/enron\_email.html}, Medical \citep{DBLP:conf/bionlp/PestianBMHJCD07}, Emotions \citep{DBLP:conf/ismir/TrohidisTKV08}, Cal500 \citep{DBLP:conf/eccv/DuyguluBFF02}, Image \citep{DBLP:journals/pr/ZhangZ07}, scene \citep{DBLP:journals/pr/BoutellLSB04}, slashdot \footnote{http://waikato.github.io/meka/datasets}.  
The datasets are collected from different domains, such as images (i.e. Corel5k, Image, scene), music (i.e. Emotions, Cal500) and text (i.e. Enron, Medical, slashdot). The statistics of these datasets can be found in Table~\ref{Tab:DataRealWorld}.
\begin{table*}[tbp] \small
	\caption{Statistics of multi-label benchmark datasets.}
	\label{Tab:DataRealWorld}
	\centering
	\begin{tabular}{l|ccccl}
		\hline
		Datasets & \#Instances & \#Features & \#Labels & \#Domain  \\
		\hline \hline
		Corel5k & 5000 & 499 & 374 &   images \\
		Emotions & 593 & 72 & 6 & music  \\
		Enron &1702  & 1001& 53 & text  \\
		Medical & 978 & 1449 & 45 & text \\
		Cal500 & 502 & 68 & 174  & music \\
		Image & 2000 & 103 & 14 & image \\
		scene & 2407 & 294 & 6 & image \\
		slashdot  & 3782 & 103 & 14 & text \\
		\hline
	\end{tabular}
\end{table*}
\begin{figure*}[!t]
	\centering
	\subfigure[CoreI5k]{
		\includegraphics[width=0.23\textwidth]{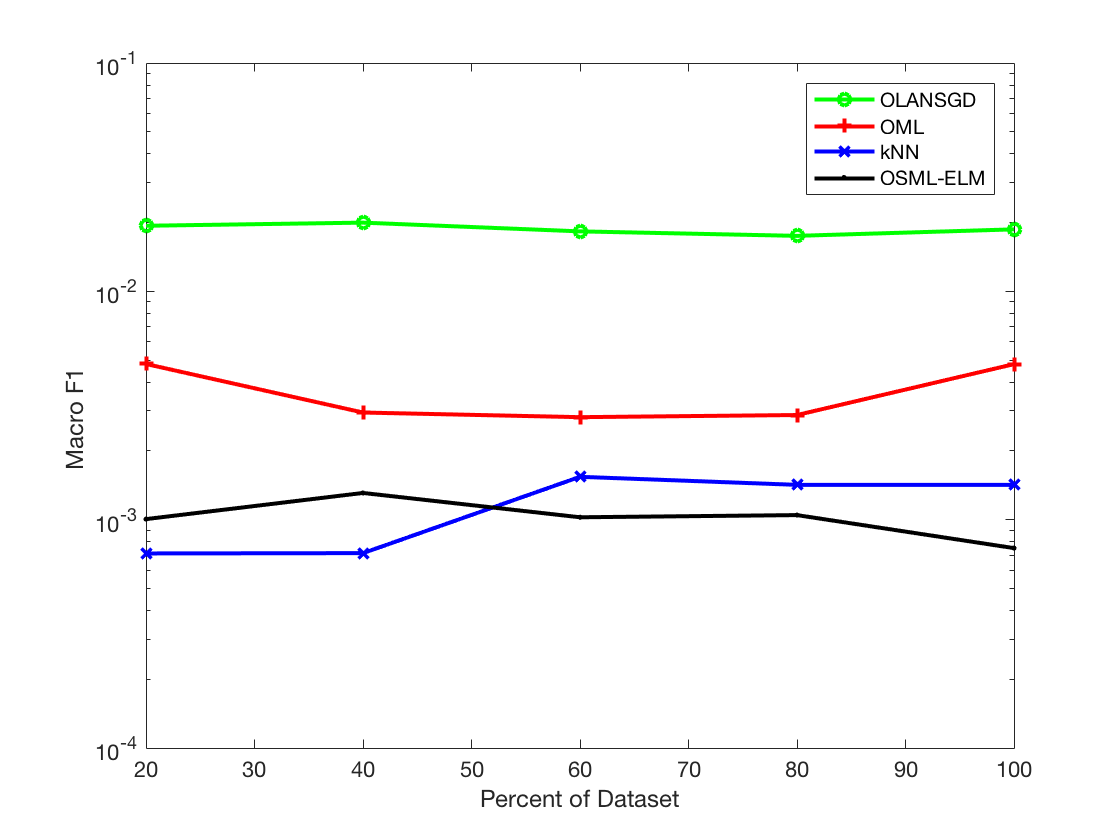}
	}
	\subfigure[Enron]{
		\includegraphics[width=0.23\textwidth]{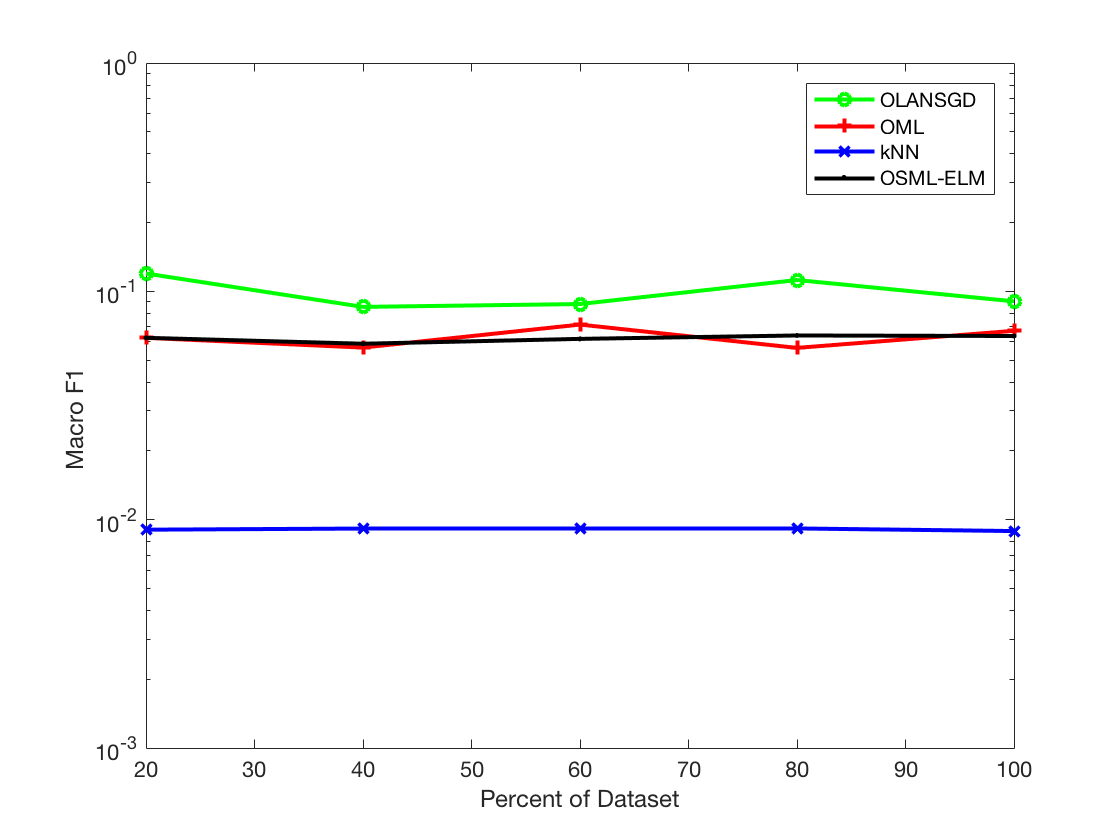}
	}
	\subfigure[FullMedical]{
		\includegraphics[width=0.23\textwidth]{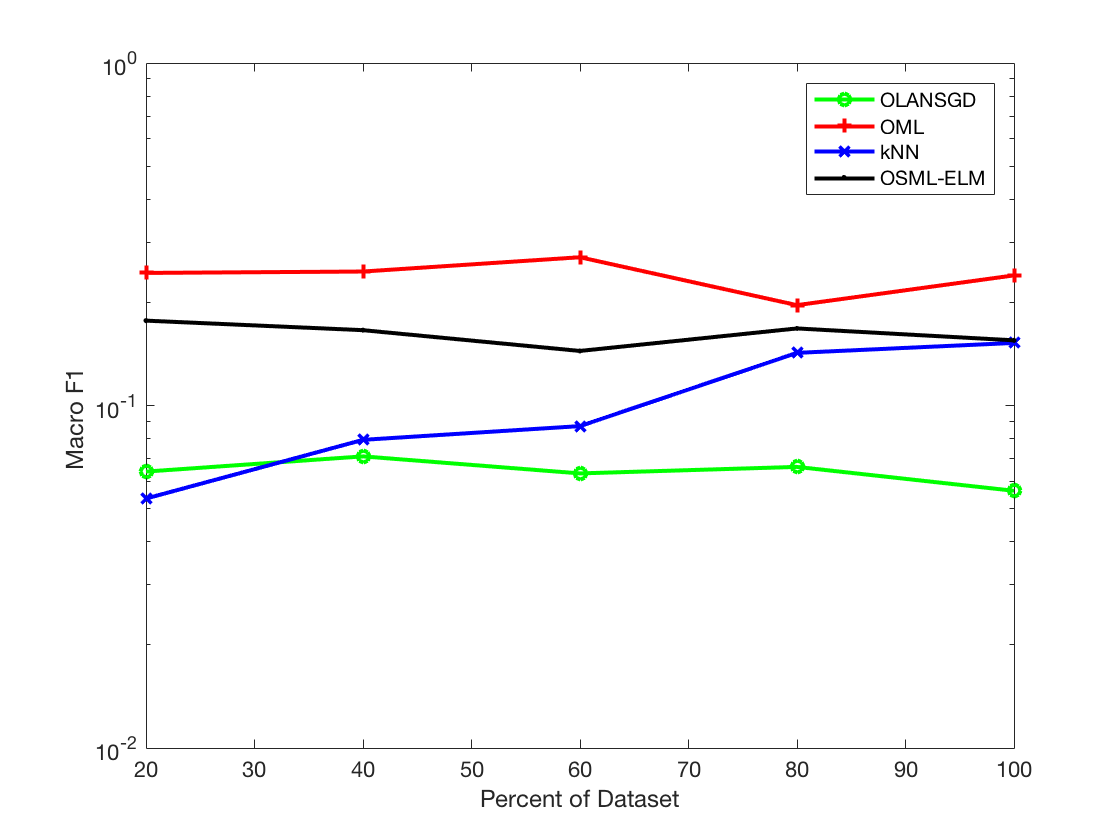}
	}
	\subfigure[Emotions]{
		\includegraphics[width=0.23\textwidth]{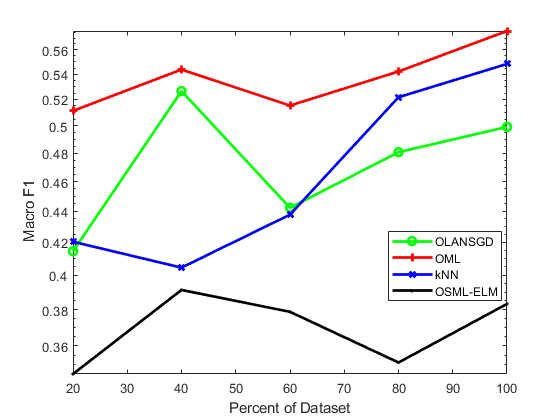}
	}
	\caption{Macro F1 of various methods on Corel5k, Enron, Medical and Emotions datasets.
	}
	\label{Experimentresults1}
\end{figure*}
\begin{figure*}[!t]
	\centering
	\subfigure[CoreI5k]{
		\includegraphics[width=0.23\textwidth]{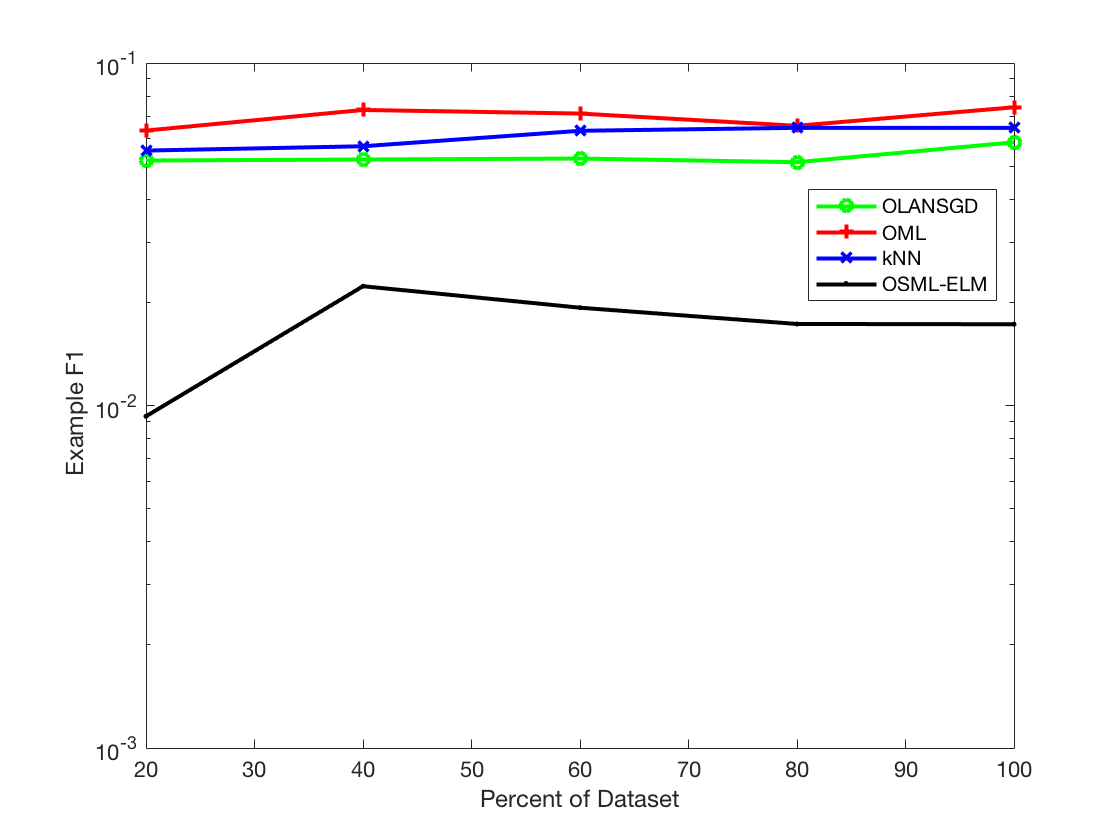}
	}
	\subfigure[Enron]{
		\includegraphics[width=0.23\textwidth]{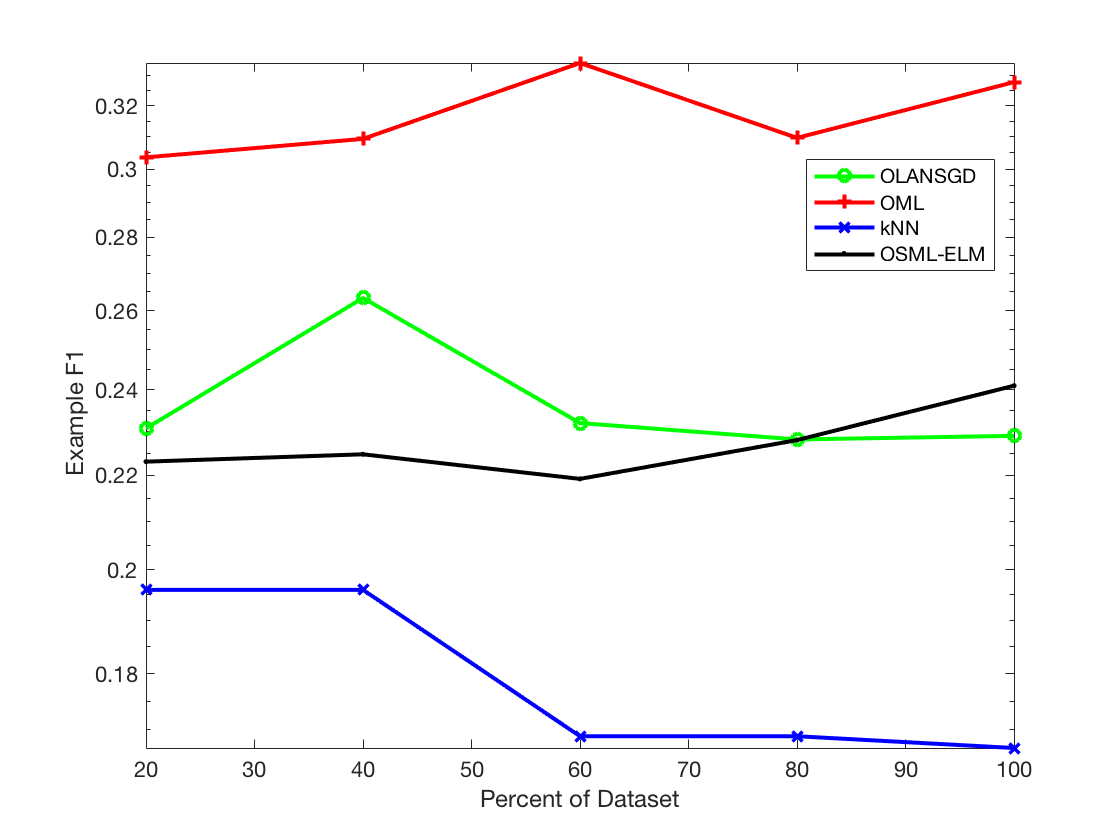}
	}
	\subfigure[FullMedical]{
		\includegraphics[width=0.23\textwidth]{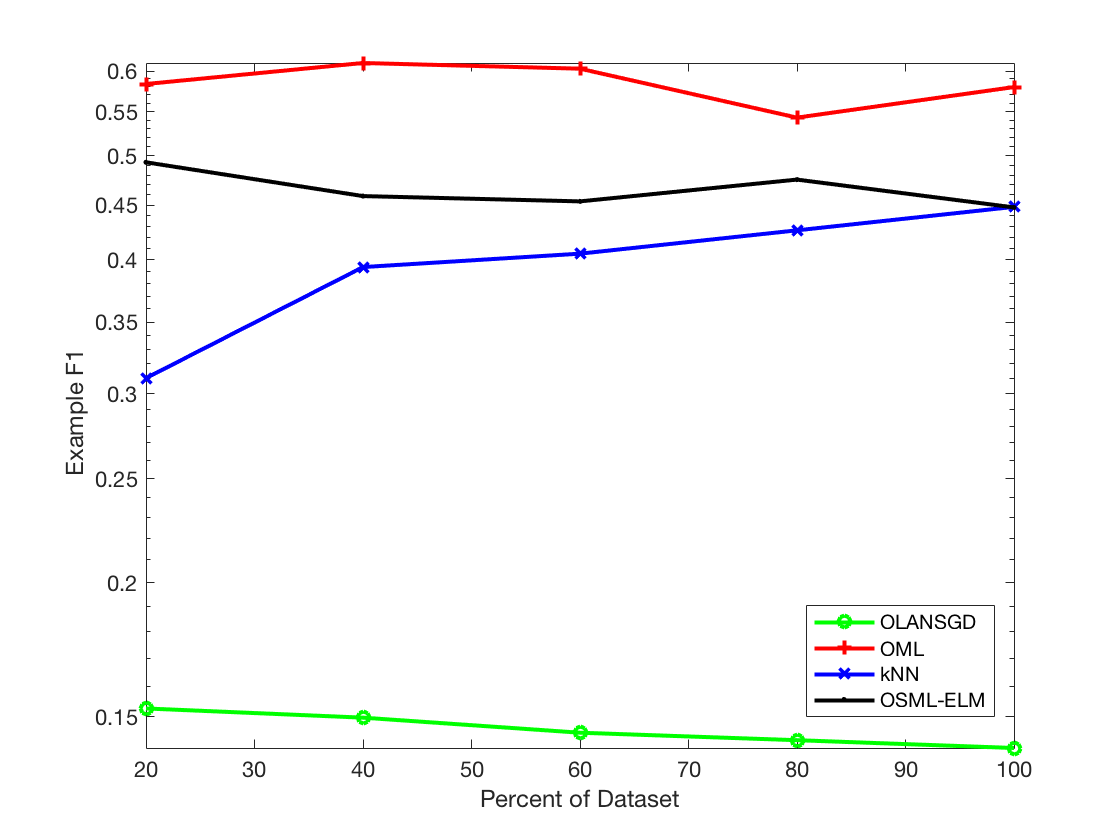}
	}
	\subfigure[Emotions]{
		\includegraphics[width=0.23\textwidth]{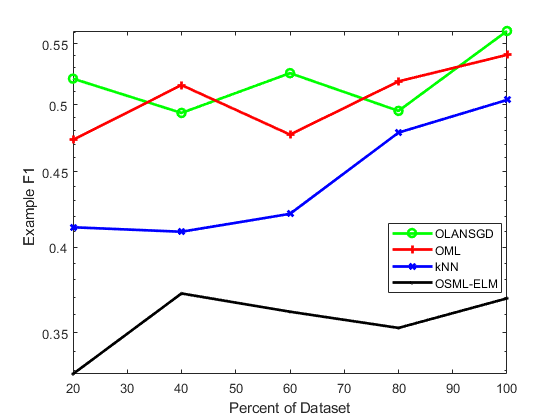}
	}
	\caption{Example F1 of various methods on Corel5k, Enron, Medical and Emotions datasets.
	}
	\label{Experimentresults2}
\end{figure*}

\begin{figure*}[!t]
	\centering
	\subfigure[CoreI5k]{
		\includegraphics[width=0.23\textwidth]{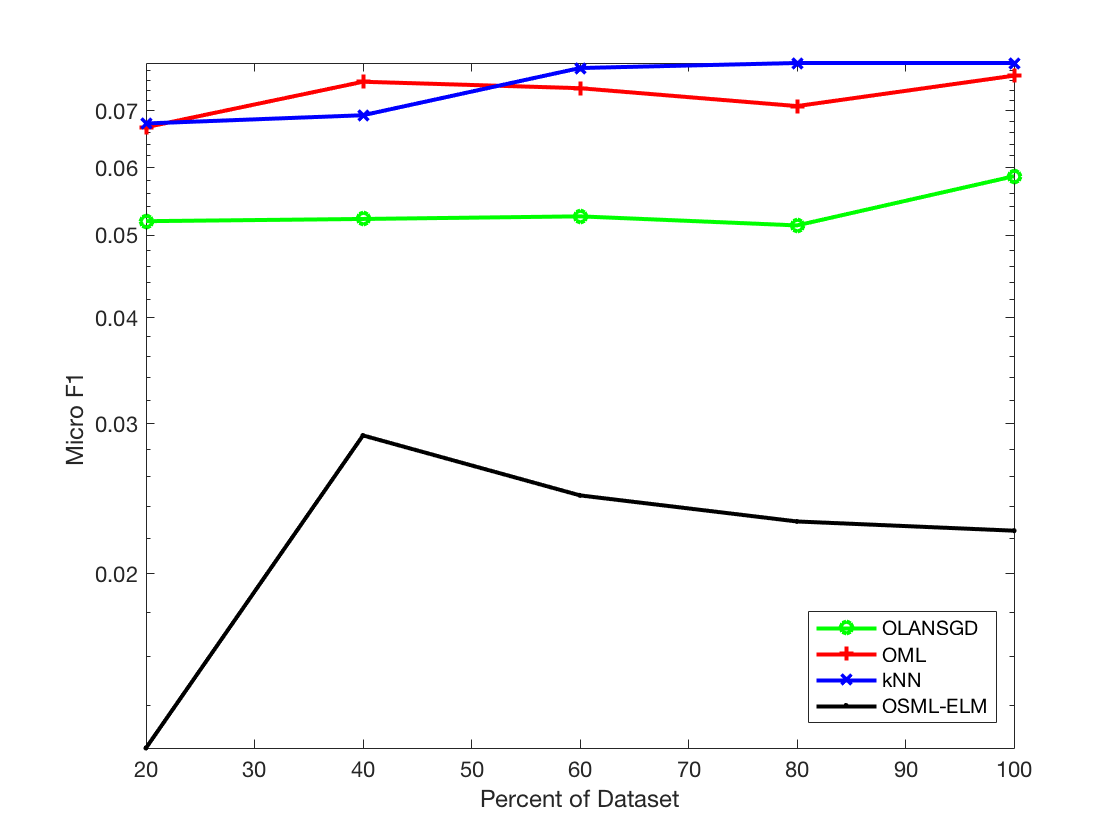}
	}
	\subfigure[Enron]{
		\includegraphics[width=0.23\textwidth]{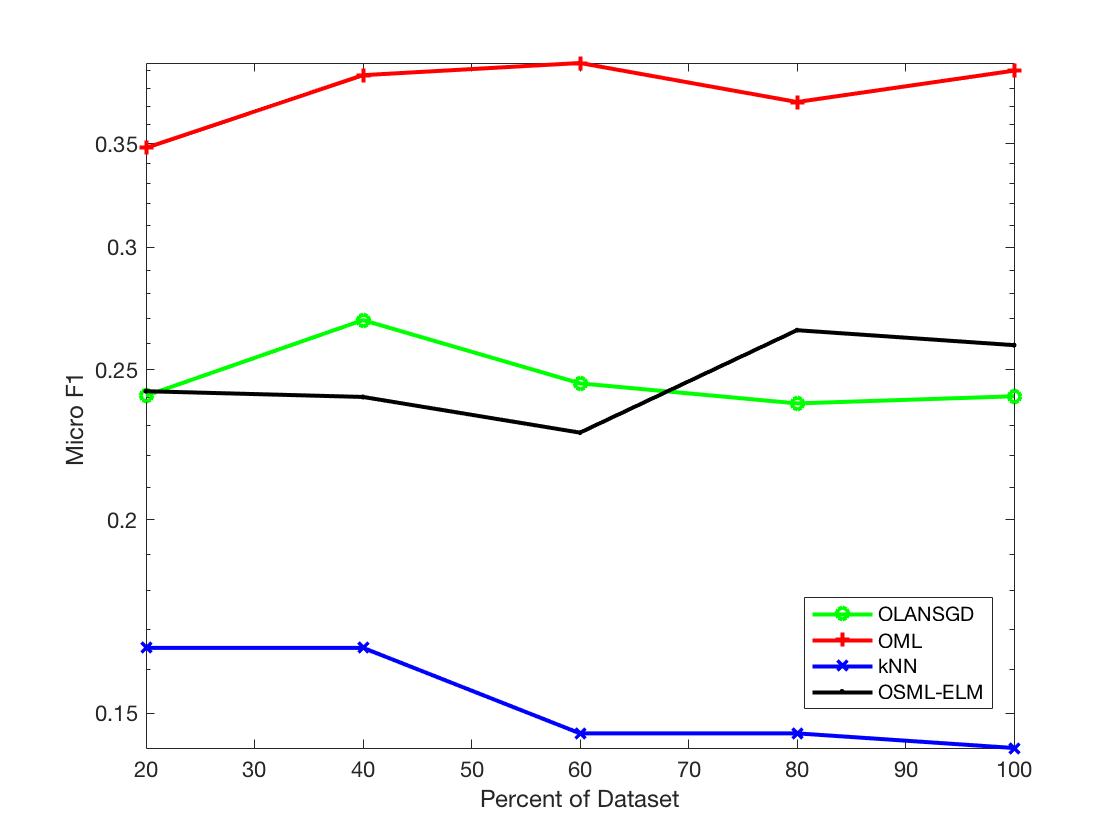}
	}
	\subfigure[FullMedical]{
		\includegraphics[width=0.23\textwidth]{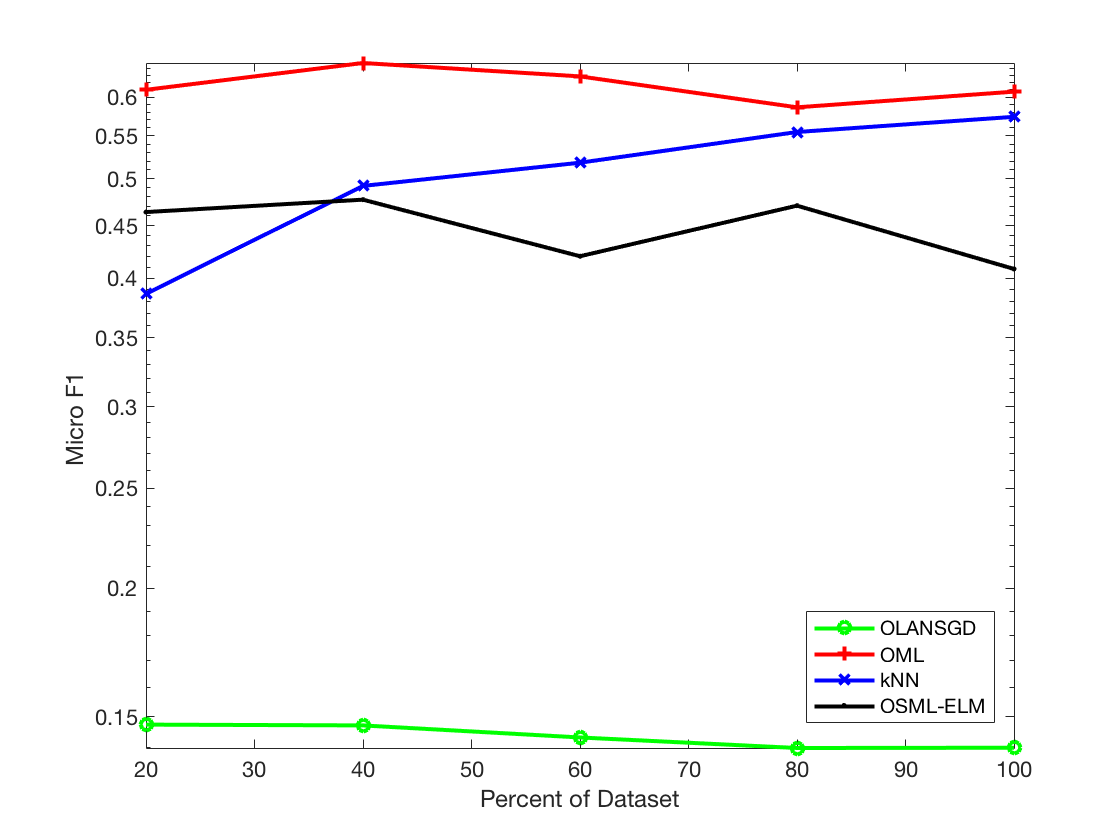}
	}
	\subfigure[Emotions]{
		\includegraphics[width=0.23\textwidth]{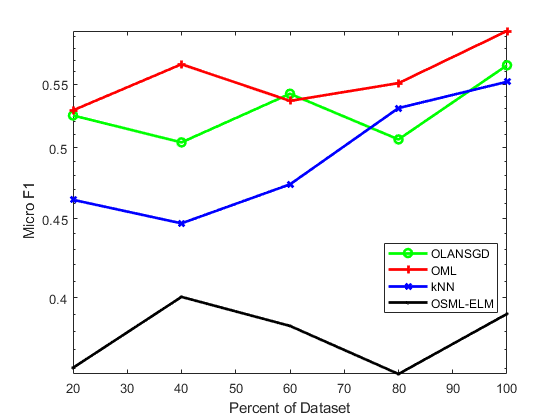}
	}
	\caption{Micro F1 of various methods on Corel5k, Enron, Medical and Emotions datasets.
	}
	\label{Experimentresults3}
\end{figure*}
\begin{figure*}[!t]
	\centering
	\subfigure[CoreI5k]{
		\includegraphics[width=0.23\textwidth]{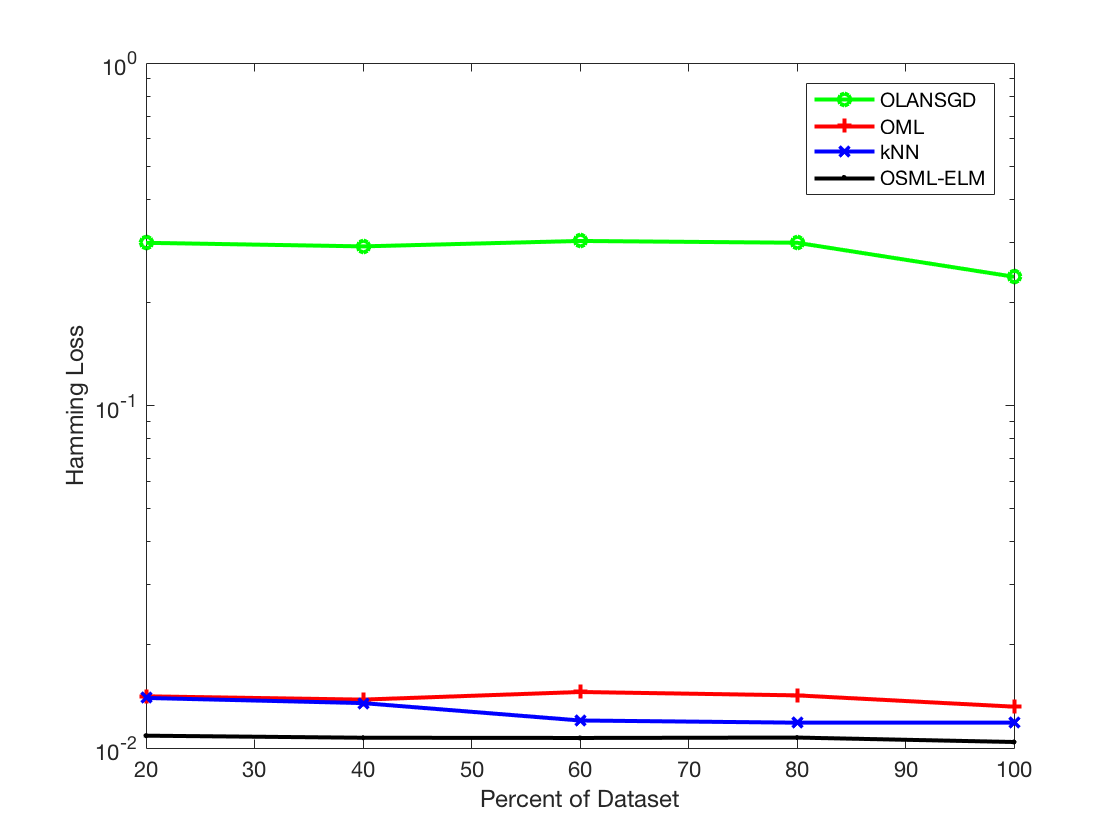}
	}
	\subfigure[Enron]{
		\includegraphics[width=0.23\textwidth]{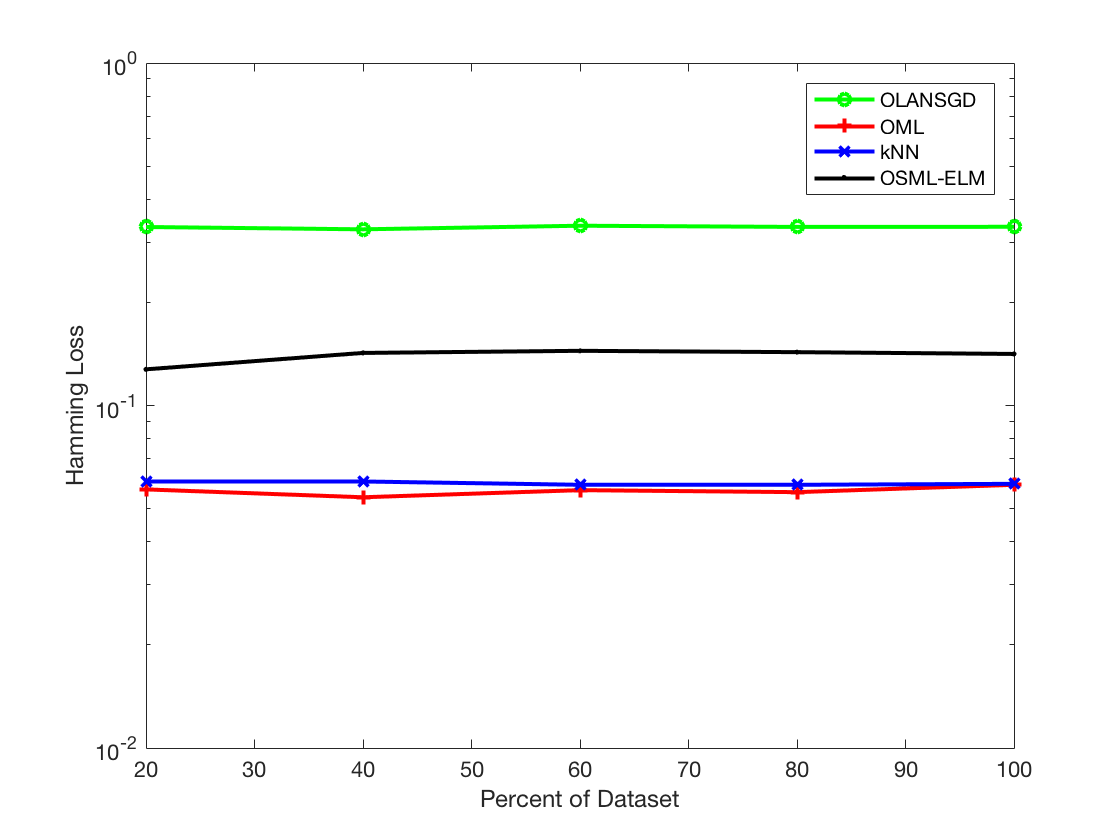}
	}
	\subfigure[FullMedical]{
		\includegraphics[width=0.23\textwidth]{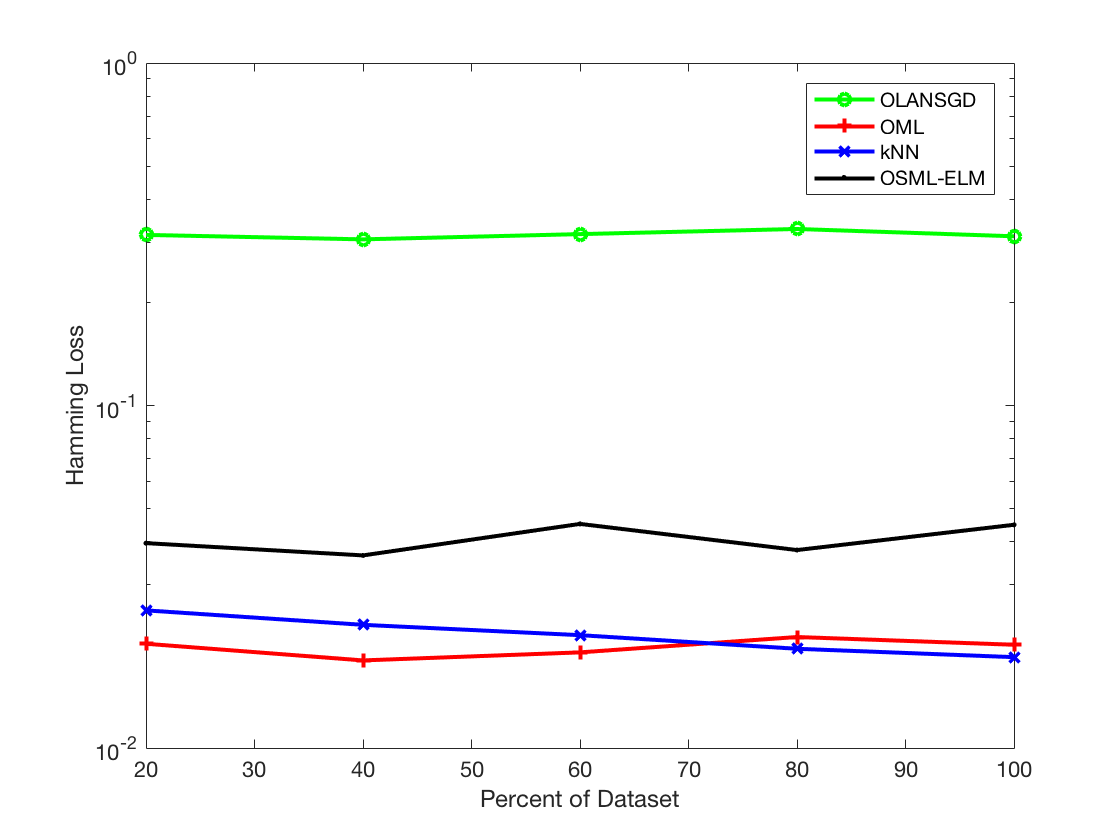}
	}
	\subfigure[Emotions]{
		\includegraphics[width=0.23\textwidth]{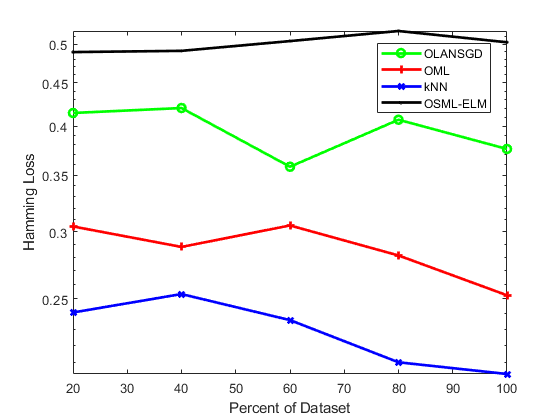}
	}
	\caption{Hamming Loss of various methods on Corel5k, Enron, Medical and Emotions datasets.
	}
	\label{Experimentresults4}
\end{figure*}

\begin{figure*}[!t]
	\centering
	\subfigure[Cal500]{
		\includegraphics[width=0.23\textwidth]{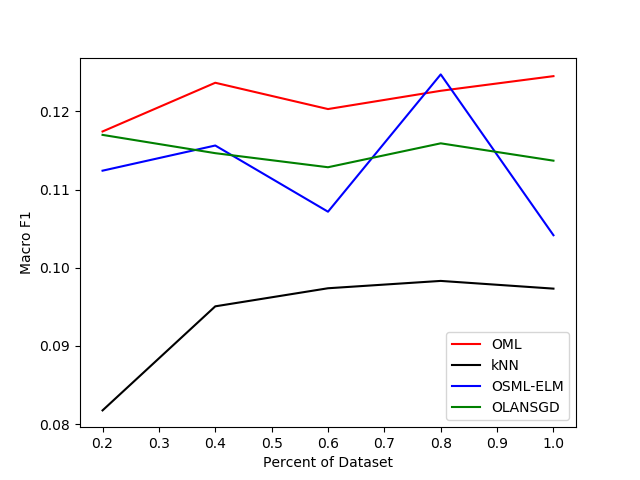}
	}
	\subfigure[Image]{
		\includegraphics[width=0.23\textwidth]{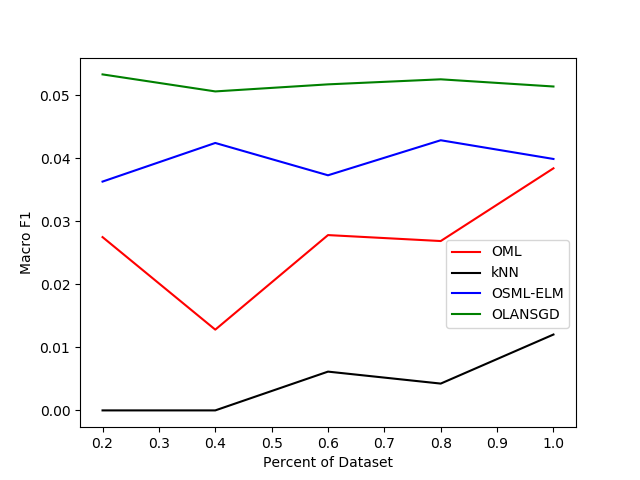}
	}
	\subfigure[scene]{
		\includegraphics[width=0.23\textwidth]{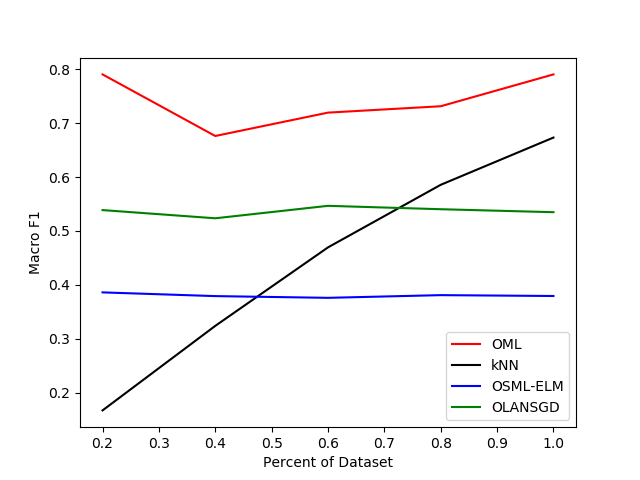}
	}
	\subfigure[slashdot]{
		\includegraphics[width=0.23\textwidth]{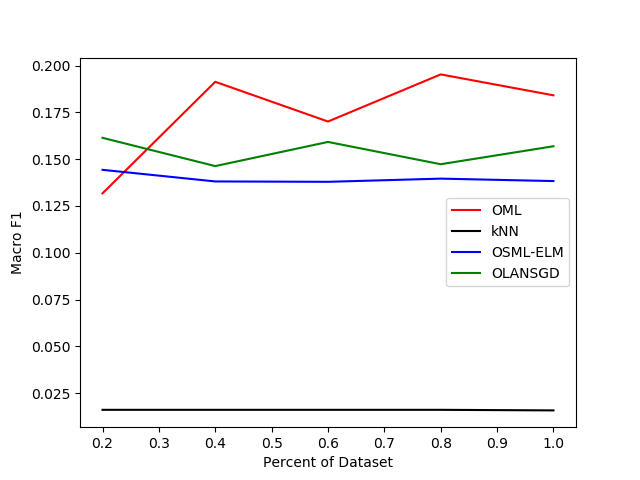}
	}
	\caption{Macro F1 of various methods on Cal500, Image, scene and slashdot datasets.
	}
	\label{Experimentresults5}
\end{figure*}

\begin{figure*}[!t]
	\centering
	\subfigure[Cal500]{
		\includegraphics[width=0.23\textwidth]{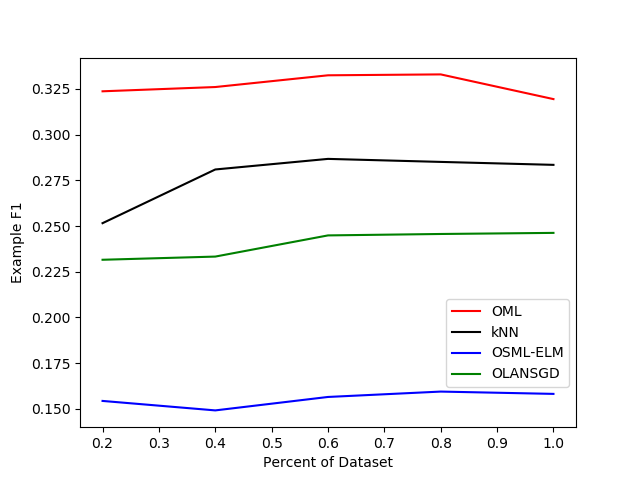}
	}
	\subfigure[Image]{
		\includegraphics[width=0.23\textwidth]{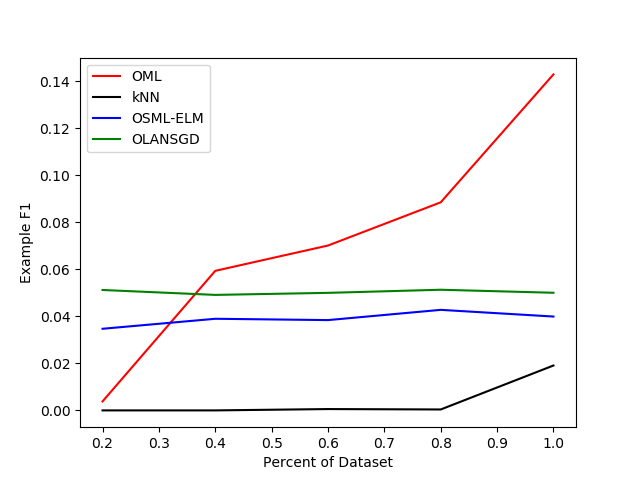}
	}
	\subfigure[scene]{
		\includegraphics[width=0.23\textwidth]{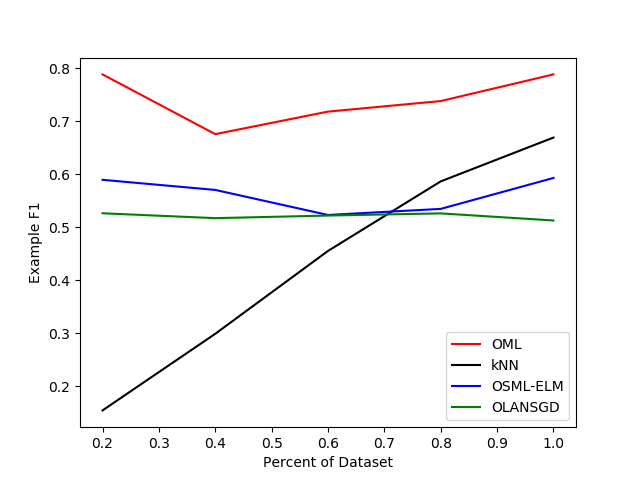}
	}
	\subfigure[slashdot]{
		\includegraphics[width=0.23\textwidth]{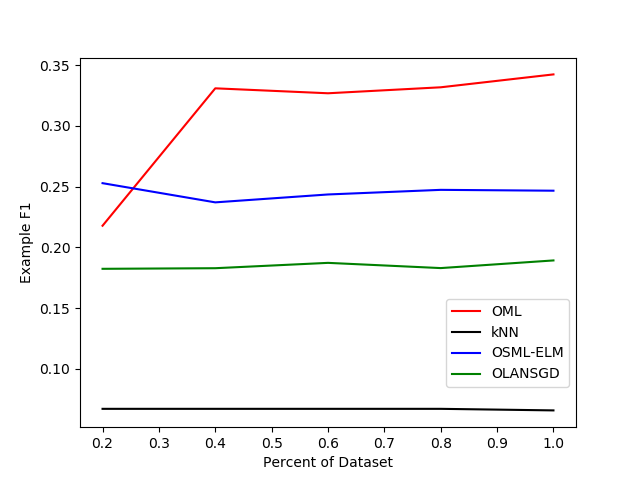}
	}
	\caption{Example F1 of various methods on Cal500, Image, scene and slashdot datasets.
	}
	\label{Experimentresults6}
\end{figure*}

\begin{figure*}[!t]
	\centering
	\subfigure[Cal500]{
		\includegraphics[width=0.23\textwidth]{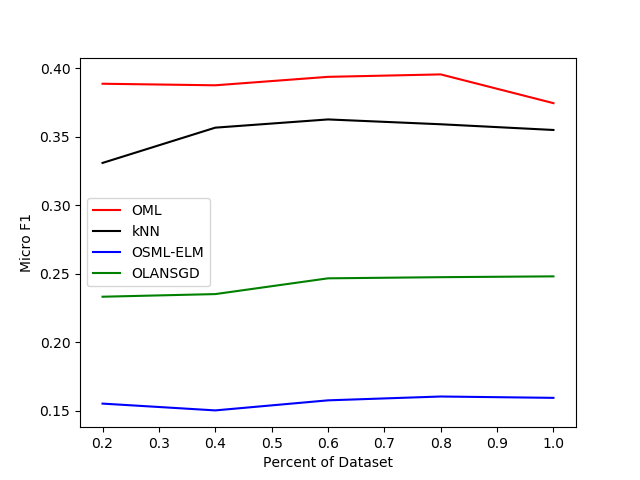}
	}
	\subfigure[Image]{
		\includegraphics[width=0.23\textwidth]{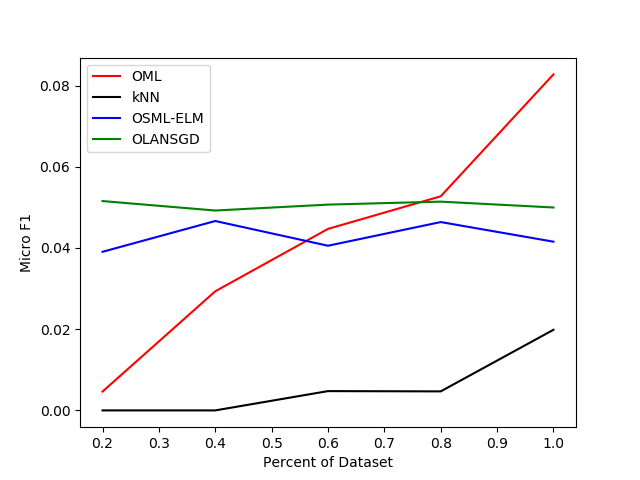}
	}
	\subfigure[scene]{
		\includegraphics[width=0.23\textwidth]{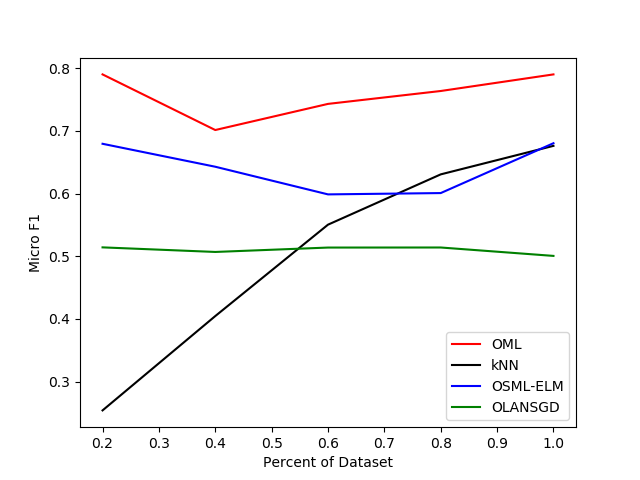}
	}
	\subfigure[slashdot]{
		\includegraphics[width=0.23\textwidth]{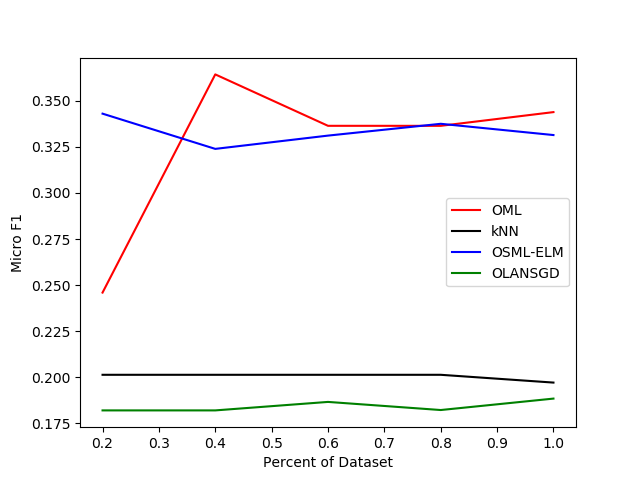}
	}
	\caption{Micro F1 of various methods on Cal500, Image, scene and slashdot datasets.
	}
	\label{Experimentresults7}
\end{figure*}
\begin{figure*}[!t]
	\centering
	\subfigure[Cal500]{
		\includegraphics[width=0.23\textwidth]{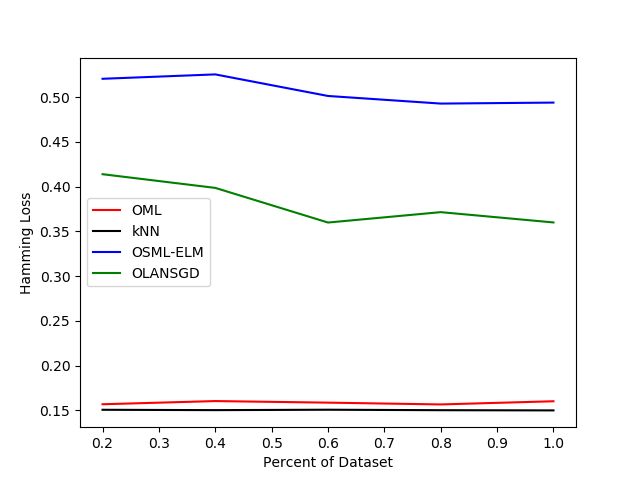}
	}
	\subfigure[Image]{
		\includegraphics[width=0.23\textwidth]{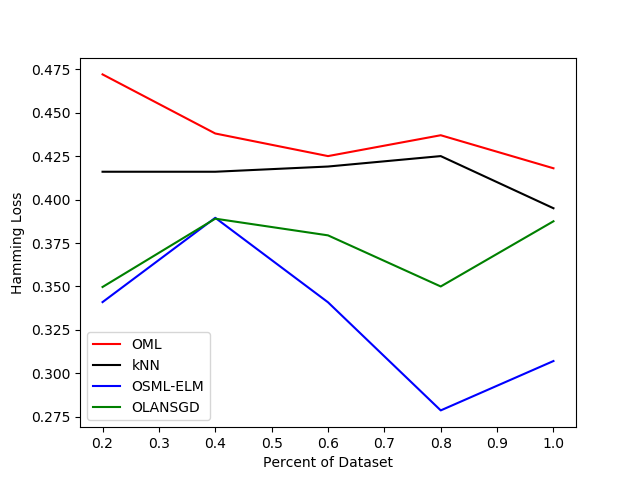}
	}
	\subfigure[scene]{
		\includegraphics[width=0.23\textwidth]{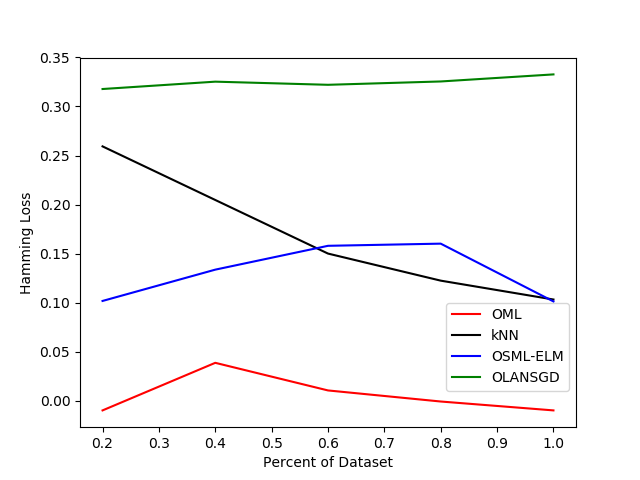}
	}
	\subfigure[slashdot]{
		\includegraphics[width=0.23\textwidth]{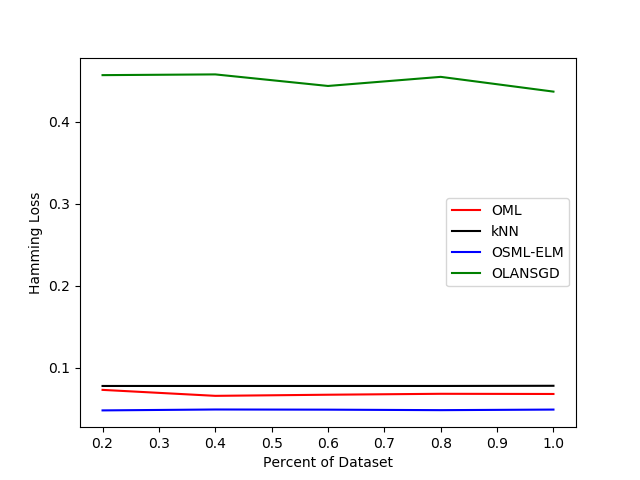}
	}
	\caption{Hamming Loss of various methods on Cal500, Image, scene and slashdot datasets.
	}
	\label{Experimentresults8}
\end{figure*}

\subsection{Experiment Setup}
\textbf{Baseline Methods}\quad We compare our OML method with several state-of-the-art online multi-label prediction methods:

\begin{itemize}
	\item OSML-ELM \cite{DBLP:journals/evs/VenkatesanEDPW17}: OSML-ELM uses a sigmoid activation function and outputs weights to predict the labels. In each step, output weight is learned from specific equation. OSML-ELM converts the label set from bipolar to unipolar representation in order to solve multi-label classification problems.
	\item OLANSGD  \cite{DBLP:conf/icassp/ParkC13}: Based on Nesterov's smooth method, OLANSGD proposes to use accelerated nonsmooth stochastic gradient descent to solve the online multi-label classification problem. It updates the model parameters using only the gradient information calculated from a single label at each iteration. It then implements a ranking function that ranks relevant and irrelevant labels.
	\item kNN: We adapt the k nearest neighbor(kNN) algorithm to solve online multi-label classification problems. A Euclidean metric is used to measure the distances between instances.
\end{itemize}
In our experiment, the matrix $V_1$ is initialized as a normal distributed random matrix. Initially, we keep 20\% of data for nearest neighbor searching. In our experiment, $M$ is set to 100000 and $m$ is set to 0.00001, while $k$ is set to 10. The codes are provided by the respective authors. Parameter $\lambda$ in OLANSGD is chosen from among $\{10^{-6}, 10^{-5},\cdots, 10^0\}$ using five-fold cross validation. We use the default parameter for OSML-ELM.

\textbf{Performance Measurements}\quad To fairly measure the performance of our method and baseline methods, we consider the following evaluation measurements \cite{journals/tip/MaoTG13,LefeiZhangDu}:
\begin{itemize}
	\item Micro-F1: computes true positives, true negatives, false positives and false negatives over all labels, then calculates an overall F-1 score.
	\item Macro-F1: calculates the F-1 score for each label, then takes the average of the F-1 score.
	\item Example-F1: computes the F-1 score for all labels of each testing sample, then takes the average of the F-1 score.
	\item Hamming Loss: computes the average zero-one score for all labels and instances.
\end{itemize}
The smaller the Hamming Loss value, the better the performance; moreover, the larger the value of the other three measurements, the better the performance.

\subsection{Prediction Performance}
Figures \ref{Experimentresults1} to \ref{Experimentresults8} present the four measurement results for our method and baseline approaches in respect of various datasets. From these figures, we can see that:
\begin{itemize}
	\item OML outperforms OSML-ELM and OLANSGD on most datasets, this is because neither of the latter approaches consider the label dependency.
	\item OML achieves better performance than $k$NN on all datasets except on Cal500 under Hmming Loss but they are comparable. This result illustrates that our proposed method is able to learn an appropriate metric for online multi-label classification.
	\item Moreover, $k$NN is comparable to OSML-ELM and OLANSGD on most datasets, which demonstrates the competitive performance of $k$NN.
\end{itemize}
Our experiments verify our theoretical studies and the motivation of this work: in short, our method is able to capture the interdependencies among labels, while also overcoming the bottleneck of $k$NN.

\section{Conclusion}
\label{Conclusion}
Current multi-label classification methods assume that all data are available in advance for leaning. Unfortunately, this assumption hinders off-line multi-label methods from handling sequential data. OLANSGD and OSML-ELM have overcome this limitation and achieved promising results in online multi-label classification; however, these methods lack a theoretical analysis for their loss functions, and also do not consider the label dependency, which has been proven to lead to degraded performance. Accordingly, to fill the current research gap on streaming data, we here propose a novel online metric learning method for multi-label classification based on the large margin principle. We first project instances and labels into the same embedding space for comparison, then learn the distance metric by enforcing the constraint that the distance between an embedded instance and its correct label must be smaller than the distance between the embedded instance and other labels. Thus, two nearby instances from different labels will be pushed further. Moreover, we develop an efficient online algorithm for our proposed model. Finally, we also provide the upper bound of cumulative loss for our proposed model, which guarantees the performance of our method on unseen data. Extensive experiments corroborate our theoretical results and demonstrate the superiority of our method.

\footnotesize
\bibliographystyle{IEEEtranN}
\bibliography{IEEEabrv,reference}


\end{document}